\documentclass[referee]{svjour3}                     % onecolumn (standard format)

\smartqed  % flush right qed marks, e.g. at end of proof

\usepackage[english]{babel}
\usepackage{amssymb,amsmath,amsfonts}
\usepackage{stmaryrd}
\usepackage[numbers]{natbib}
\usepackage{graphicx}
\usepackage{mathptmx}
\usepackage{qtree}
\usepackage{float}
\usepackage{url}
\usepackage{tikz}

\newcommand{\Fig}[1]{Fig.~\ref{#1}}

\newcommand{\Tab}[1]{Tab.~\ref{#1}}

\newcommand{\TT}[1]{\text{\tt #1}}
\newcommand{\rank}{\mathrm{rank}}

\newcommand{\dom}{\mathrm{Dom}}
\newcommand{\spn}{\mathrm{span}}
\newcommand{\mng}[1]{\llbracket #1 \rrbracket}

\newcommand{\cat}{\mathrm{cat}}
\newcommand{\cons}{\mathrm{cons}}
\newcommand{\ext}{\mathrm{ex}}

\newcommand{\bfcat}{\mathbf{cat}}
\newcommand{\bfcons}{\mathbf{cons}}
\newcommand{\bfext}{\mathbf{ex}}

\newcommand{\1}{\,\mathbf{1}}

\newcommand{\bra}[1]{\langle #1|}
\newcommand{\ket}[1]{|#1\rangle}
\newcommand{\braket}[2]{\langle #1|#2\rangle}

\makeatletter
\@addtoreset{case}{theorem}
\@addtoreset{case}{corollary}
\makeatother

\journalname{Cognitive Computation}
\begin{document}

%---------------------------------------------title page--------------------------------------------

\title{Vector symbolic architectures for context-free grammars}
\titlerunning{VSA for CFG}

\author{Peter beim Graben,$^*$\thanks{$^*$Corresponding author}
    Markus Huber, Werner Meyer, \\
    Ronald R\"omer and Matthias Wolff}

\authorrunning{beim Graben et al.}

\institute{
    Peter beim Graben
        \at
    Bernstein Center for Computational Neuroscience, Berlin, Germany
    \and
    Peter beim Graben \and Markus Huber \and Werner Meyer \and Ronald R\"omer \and Matthias Wolff
        \at
    Department of Communication Engineering \\
    Brandenburgische Technische Universit\"at (BTU) Cottbus--Senftenberg\\
    Platz der Deutschen Einheit 1 \\
    D -- 03046 Cottbus\\
    \email{peter.beimgraben@b-tu.de}
    }

\date{\today}
\maketitle

\begin{abstract}
Background / introduction. Vector symbolic architectures (VSA) are a viable approach for the hyperdimensional representation of symbolic data, such as documents, syntactic structures, or semantic frames.
Methods. We present a rigorous mathematical framework for the representation of phrase structure trees and parse trees of context-free grammars (CFG) in Fock space, i.e. infinite-dimensional Hilbert space as being used in quantum field theory. We define a novel normal form for CFG by means of term algebras. Using a recently developed software toolbox, called FockBox, we construct Fock space representations for the trees built up by a CFG left-corner (LC) parser.
Results. We prove a universal representation theorem for CFG term algebras in Fock space and illustrate our findings through a low-dimensional principal component projection of the LC parser states.
Conclusions. Our approach could leverage the development of VSA for explainable artificial intelligence (XAI) by means of hyperdimensional deep neural computation. It could be of significance for the improvement of cognitive user interfaces and other applications of VSA in machine learning.
\end{abstract}


% ------------------------------------- Section -----------------------------------
\section{Introduction}
\label{sec:intro}

Claude E. Shannon, the pioneer of information theory, presented in 1952 a ``maze-solving machine'' as one of the first proper technical cognitive systems \cite{Shannon53}.\footnote{
See also Shannon's instructive video demonstration at \url{https://www.youtube.com/watch?v=vPKkXibQXGA}.
}
It comprises a maze in form of a rectangular board partitioned into discrete cells that are partially separated by removable walls, and a magnetized ``mouse'' (nicknamed ``Theseus'', after the ancient Greek hero) as a cognitive agent. The mouse possesses as an actuator a motorized electromagnet beneath the maze board. The magnet pulls the mouse through the maze. Sensation and memory are implemented by a circuit of relays, switching their states after encounters with a wall. In this way, Shannon technically realized a
simple, non-hierarchic \emph{perception-action cycle} (PAC) \cite{Young10}, quite similar to the more sophisticated version depicted in \Fig{fig:PAC} as a viable generalization of a cybernetic feedback loop.

In general, PAC form the core of a \emph{cognitive dynamic system} \cite{Young10, Haykin12}. They describe the interaction of a cognitive agent with a dynamically changing world as shown in \Fig{fig:PAC}. The agent is equipped with \emph{sensors} for the perception of its current state in the environment and with \emph{actuators} allowing for active state changes. A central control prescribes goals and strategies for problem solving that could be trained by either trial-and-error learning as in Shannon's construction, or, more generally, by reinforcement learning \cite{Haykin12}.

\begingroup
\newcommand{\BOX}[3]{%
  \pgfmathsetmacro\x{(#1)}
  \pgfmathsetmacro\y{(#2)}
  \def\w{4}
  \def\h{1}
  \pgfmathsetmacro\mpw{(\xy*\w)}
  \draw[rounded corners] (\x-\w/2, \y-\h/2) rectangle (\x+\w/2, \y+\h/2);
  \node at (\x,\y) {\begin{minipage}{\mpw cm}\centering\baselineskip=0pt #3\end{minipage}};
}
\def\xy{0.65}
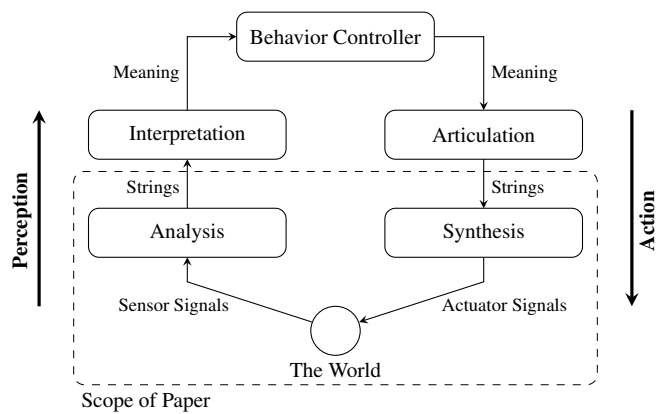
\begin{figure}[H]
\begin{center}
  \begin{tikzpicture}[x=\xy cm, y=\xy cm, >=stealth]
    \draw (3,0) circle (0.5);
    \node[anchor=north] at (3,-0.5) {The World};
    \draw[->] (2.52,0.18) -- (0,1) -- (0,1.5); \node[anchor=east] at (1,0.5) {\scriptsize Sensor Signals};
    \draw[->] (6,1.5) -- (6,1) -- (3.48,0.18); \node[anchor=west] at (5,0.5) {\scriptsize Actuator Signals};
    \BOX{0}{2}{Analysis};
    \BOX{6}{2}{Synthesis};
    \draw[->] (0,2.5) -- (0,3.5); \node[anchor=east] at (0,2.9) {\scriptsize Strings};
    \draw[->] (6,3.5) -- (6,2.5); \node[anchor=west] at (6,2.9) {\scriptsize Strings};
    \BOX{0}{4}{Interpretation};
    \BOX{6}{4}{Articulation};
    \draw[->] (0,4.5) -- (0,6) -- (1,6); \node[anchor=east] at (0,5.25) {\scriptsize Meaning};
    \draw[->] (5,6) -- (6,6) -- (6,4.5); \node[anchor=west] at (6,5.25) {\scriptsize Meaning};
    \BOX{3}{6}{Behavior Controller};
    \draw[rounded corners,very thin,dashed] (-2.3,-1.1) rectangle (8.3,3.25);
    \node[anchor=north west] at (-2.3,-1.1) {Scope of Paper};
    \draw[very thick,->] (-3,0.5) -- (-3,4.5); \node[rotate=90,anchor=south] at (-3,2.25) {\bf Perception};
    \draw[very thick,->] ( 9,4.5) -- ( 9,0.5); \node[rotate=90,anchor=north] at ( 9,2.25) {\bf Action};
  \end{tikzpicture}
  \caption{Hierarchical perception-action cycle (PAC) for a cognitive dynamic system. The scope of the present paper is indicated by the dashed boundary.}
  \label{fig:PAC}
\end{center}
\end{figure}
\endgroup

In Shannon's mouse-maze system, the motor (the \emph{actuator}) pulls the mouse along a path until it bumps into a wall which is registered by a \emph{sensor}. This perception is stored by switching a relay, subsequently avoiding the corresponding action. The \emph{behavior control} prescribes a certain maze cell where the agent may find a ``piece of cheese'' as a goal. When the goal is eventually reached, no further action is necessary. In a first run, the mouse follows an irregular path according to a trial-and-error strategy, while building up a memory trace in the relay array. In every further run, the successfully learned path is pursued at once. However, when the operator modifies the arrangement of walls, the previously learned path becomes useless and the agent has to learn from the very beginning. Therefore, \citet[p. 1238]{Shannon53} concludes:
\begin{quote}
    The maze-solver may be said to exhibit at a very primitive level the abilities to (1) solve problems by trial and error, (2) repeat the solutions without the errors, (3) add and correlate new information to a partial solution, (4) forget a solution when it is no longer applicable.
\end{quote}

In Shannon's original approach, the mouse learns by trial-and-error whenever it bumps into a wall. More sophisticated cognitive dynamic systems should be able to draw logical inferences and to communicate either with each other or with an external operator, respectively \cite{RomerEA19}. This requires higher levels of mental representations such as formal logics and grammars. Consider, e.g., the operator's utterance:
\begin{equation}\label{eq:sentence}
  \TT{the mouse ate cheese}
\end{equation}
(note that symbols will be set in typewriter font in order to abstract from their conventional meaning in the first place). In the PAC described in \Fig{fig:PAC}, the acoustic signal has firstly to be \emph{analyzed} in order to obtain a \emph{phonetic string representation}. For understanding its \emph{meaning}, the agent has secondly to process the utterance grammatically through \emph{syntactic parsing}. Finally, the syntactic representation, e.g. in form of a phrase structure tree, must be \emph{interpreted} as a semantic representation which the agent can ultimately understand \cite{Karttunen84}. Depending upon such understanding, the agent can draw logical inferences and derive the appropriate behavior for controlling the actuators. In case of verbal behavior \cite{Skinner15}, the agent therefore computes an appropriate response, first as a semantic representation, that is \emph{articulated} into a syntactic and phonetic form and finally \emph{synthesized} as an acoustic signal. In any case, high-level representations are symbolic and their processing is rule-driven, in contrast to low-level sensation and actuation where physical signals are essentially continuous.

Originally, Shannon used an array of relays as the agent's memory. This has later been termed the ``learning matrix'' by \citet{SteinbuchSchmitt67}. Learning matrices and \emph{vector symbolic architectures} (VSA) provide viable interfaces between hierarchically organized symbolic data structures such as phrase structure trees or semantic representations and continuous state space approaches as required for deep neural networks (DNN) \cite{CunBengioHinton15, Schmidhuber15}. Beginning with seminal studies by \citet{Smolensky90} and \citet{Mizraji89}, and later pursued by \citet{Plate95}, \citet{GrabenPotthast09a}, and \citet{Kanerva09} among many others, those architectures have been dubbed VSA by \citet{Gayler06} (cf. also \cite{LevyGayler08}).

In a VSA, symbols and variables are represented as filler and role vectors of some underlying linear \emph{embedding spaces} \cite{BengioCourvilleVincent13, JonesMewhort07}, respectively. When a symbol is assigned to a variable, the corresponding filler vector is \emph{bound} to the corresponding role vector. Different filler-role bindings can be \emph{bundled} together to form a data structure \cite{LevyGayler08}, such as a list, a frame, or a table of a relational data base \cite{SchmittWirschingWolff19}. Those structures can be recursively bound to other fillers and further bundled together to yield arbitrarily complex data structures \cite{GrabenPotthast09a}.

VSA have recently been employed for semantic spaces \cite{JonesMewhort07, RecchiaEA15}, logical inferences \cite{EmruliGaylerSandin13, WiddowsCohen14, Mizraji20}, data base queries \cite{SchmittWirschingWolff19, KleykoOsipovGayler16}, and autoassociative memories \cite{GritsenkoEA17, MizrajiPomiLin18}. \Citet{WolffHuberEA18} developed a VSA model for cognitive representations and their induction in Shannon's mouse-maze system. In the present study, we focus on the dashed region in \Fig{fig:PAC}, by elaborating earlier approaches for VSA language processors \cite{GrabenPotthast09a, CarmantiniEA17}. Specifically, we discuss vector space representations of context-free grammars (CFG) and push-down automata \cite{HopcroftUllman79}, as used in current speech and language technologies \cite{BengioCourvilleVincent13, OtterMedinaKalita20, Goldberg17, MinaeeEA20}.

Deploying neural networks in language technology became increasingly important in recent time. Beginning with hard-wired recurrent neural architectures \cite{ChenHonavar99, Pollack91, SiegelmannSontag95, CarmantiniEA17}, the advent of deep learning algorithms lead to  state-of-the-art language processing through recursive neural networks (RNN, \cite{SocherManningNg10}), through long-short-term memory networks (LSTM, \cite{HochreiterSchmidhuber97, HupkesDankersEA20}), and through convolutional neural networks (CNN, \cite{CunBengioHinton15, DauphinFanEA16}), with their most recent improvements, capsule networks \cite{PatrickEA19, YangZhao19}; for a survey consult \cite{BengioCourvilleVincent13, OtterMedinaKalita20, Goldberg17, MinaeeEA20}. Particularly interesting are latest attempts of Smolensky and collaborators to merge VSA and DNN into tensor product recurrent networks (TPRN, \cite{PalangiSmolenskyEA17,  PalangiSmolenskyEA18, TangSmolenskySa19}) which are able to directly learn filler-role bindings by end-to-end training under a special quantization regularization constraint.

Despite these impressive achievements, DNN are intrinsic black-box models, propagating input patterns through their hidden layers toward the associated output patterns. The hidden layers may have several hundred-thousands up to some billions synaptic weight parameters that are trained by regularized gradient climbing algorithms. After training, the network develops a hidden representation of the input features and the computational rules to transform them into output. Yet these representations are completely opaque and nobody can \emph{explain} how the input is mapped onto the output \cite{CunBengioHinton15}.

Therefore, according to \citet{Marcus20}, the next-generation AI, must be explainable, robust and trustworthy. Creating \emph{explainable AI} (XAI) \cite{DoranSchulzBesold17} is an important challenge for current research \cite{MontavonSamekMuller18}. For this aim, it is mandatory not only to develop new algorithms and networks architectures, such as TPRN \cite{PalangiSmolenskyEA17,  PalangiSmolenskyEA18, TangSmolenskySa19}, e.g., but also conceptual understanding of their formal structures. To this end, we present rigorous proofs for vector space representations of context-free grammars (CFG) and push-down automata. We suggest a novel normal form for CFG, allowing to express CFG parse trees as terms over a symbolic term algebra. Rule-based derivations over that algebra are then represented as transformation matrices in Fock space \cite{Fock32, Aerts09}. Our approach could lead to the development of new machine learning algorithms for training neural networks as rule-based symbol processors. In contrast to black-box DNN, our method is essentially transparent and hence explainable and trustworthy.

% ------------------------------------- Section -----------------------------------
\section{Methods}
\label{sec:meth}

We start from a symbolic, rule-based system that can be described in terms of formal grammar and automata theory. Specifically, we chose context-free grammars (CFG) and push-down automata as their processors here \cite{HopcroftUllman79}. In the second step, we reformulate these languages through term algebras and their processing through partial functions over term algebras. We introduce a novel normal form for CFG, called \emph{term normal form}, and prove that any CFG in Chomsky normal form can be transformed into term normal form. Finally, we introduce a vector symbolic architecture by assigning basis vectors of a high-dimensional linear space to the respective symbols and their roles in a phrase structure tree. We suggest a recursive function for mapping CFG phrase structure trees onto representation vectors in Fock space and prove a representation theorem for the partial rule-based processing functions. Finally, we present a software toolbox, FockBox for handling Fock space VSA representations \cite{WolffWirschingEA18}.

% ------------------------------------- Section -----------------------------------
\subsection{Context-free Grammars}
\label{sec:cfg}

Consider again the simple sentence \eqref{eq:sentence} as a motivating example. According to linguistic theory, sentences such as \eqref{eq:sentence} exhibit a hierarchical structure, indicating a logical subject-predicate relationship. In \eqref{eq:sentence} ``\TT{the mouse}'' appears as subject and the phrase ``\TT{ate cheese}'' as the predicate, which is further organized into a transitive verb ``\TT{ate}'' and its direct object ``\TT{cheese}''. The hierarchical structure of sentence \eqref{eq:sentence} can therefore be either expressed through regular brackets, as in \eqref{eq:braksen}
\begin{equation}\label{eq:braksen}
  \TT{[[[the] [mouse]] [ate [cheese]]]} \:,
\end{equation}
or, likewise as a phrase structure tree as in \Fig{fig:pst}

\begin{figure}[H]
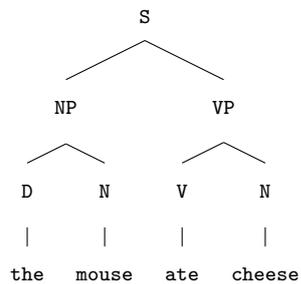

  \centering
    \Tree
    [.\TT{S}
        [.\TT{NP}
            [.\TT{D} \TT{the}
            ].\TT{D}
            [.\TT{N} \TT{mouse}
            ].\TT{N}
        ].\TT{NP}
        [.\TT{VP}
            [.\TT{V} \TT{ate}
            ].\TT{V}
            [.\TT{N} \TT{cheese}
            ].\TT{N}
        ].\TT{VP}
    ].\TT{S}
  \caption{\label{fig:pst} Phrase structure tree of example sentence \eqref{eq:sentence}.}
\end{figure}

In \Fig{fig:pst} every internal node of the tree denotes a syntactic category: \TT{S} stands for ``sentence'', \TT{NP} for ``noun phrase'', the sentence's subject, \TT{VP} for ``verbal phrase'', the predicate, \TT{D} for ``determiner'', \TT{N} for ``noun'', and \TT{V} for ``verb''.

The phrase structure tree \Fig{fig:pst} immediately gives rise to a context-free grammar (CFG) by interpreting every branch as a rewriting rule in Chomsky normal form \cite{Kracht03, HopcroftUllman79}

\begin{subequations}
\begin{align}
  \TT{S} &\to \TT{NP} \quad \TT{VP} \label{eq:cfg1} \\
  \TT{NP} &\to \TT{D} \quad \TT{N} \label{eq:cfg2} \\
  \TT{VP} &\to \TT{V} \quad \TT{N} \label{eq:cfg3} \\
  \TT{D} &\to \TT{the} \label{eq:cfg4} \\
  \TT{N} &\to \TT{mouse} \label{eq:cfg5} \\
  \TT{V} &\to \TT{ate} \label{eq:cfg6} \\
  \TT{N} &\to \TT{cheese} \label{eq:cfg7}
\end{align}
\end{subequations}
where one distinguishes between syntactical rules (\ref{eq:cfg1} -- \ref{eq:cfg3}) and lexical rules (\ref{eq:cfg4} -- \ref{eq:cfg7}), respectively. More abstractly, a CFG is given as a quadruple $G = (T, N, \mathtt{S}, R)$, such that in our example $T = \{ \TT{the},  \TT{mouse}, \TT{ate}, \TT{cheese} \}$ is the set of words or terminal symbols, $N = \{ \TT{S}, \TT{NP}, \TT{VP}, \TT{D},  \TT{N}, \TT{V} \}$ is the set of categories or nonterminal symbols, $\TT{S} \in N$ is the distinguished start symbol, and $R \subset N \times (N \cup T)^*$ is a set of rules. A rule $r = (A, \gamma) \in R$ is usually written as a production $r : A \to \gamma$ where $A \in N$ denotes a category and $\gamma \in (N \cup T)^* $ a finite string of terminals or categories of length $n = |\gamma|$.

Context-free grammars can be processed by push-down automata \cite{HopcroftUllman79}. Regarding psycholinguistic plausibilty, the left-corner (LC) parser is particularly relevant because input-driven bottom-up and expectation-driven top-down processes are tightly intermingled with each other \cite{Hale11}. An LC parser possesses, such as any other push-down automaton, two memory tapes: firstly a working memory, called stack, operating in a last-in-first-out (LIFO) fashion, and an input tape storing the sentence to be processed.

In the most simple cases, when a given CFG does not contain ambiguities (as in (\ref{eq:cfg1} -- \ref{eq:cfg7}) for our example \eqref{eq:sentence}), an LC parser can work deterministically. The LC parsing algorithm operates in four different modes: i) if nothing else is possible and if the input tape is not empty, the first word of the input is \emph{shift}ed into the stack; ii) if the first symbol in the stack is the left corner of a syntactic rule, the first stack symbol is rewritten by a predicted category (indicated by square brackets in \Tab{tab:lcpars}) followed by the left-hand side of the rule (\emph{project}); iii) if a category in the stack was correctly predicted, the matching symbols are removed from the stack (\emph{complete}); iv) if the input tape is empty and the stack only contains the start symbol of the grammar, the automaton moves into the \emph{accept}ing state; otherwise, syntactic language processing had failed. Applying the LC algorithm to our example CFG leads to the symbolic process shown in \Tab{tab:lcpars}.

\begin{table}[H]
  \centering
  \begin{tabular}{llll}
     \hline
     step & stack & input & operation \\
     \hline
     0 & $\epsilon$ & \TT{the mouse ate cheese} & shift \\
     1 & \TT{the} & \TT{mouse ate cheese} & project \eqref{eq:cfg4} \\
     2 & \TT{D} & \TT{mouse ate cheese} & project \eqref{eq:cfg2} \\
     3 & \TT{[N]} \TT{NP}  & \TT{mouse ate cheese} & shift \\
     4 & \TT{mouse} \TT{[N]} \TT{NP}  & \TT{ate cheese} & project \eqref{eq:cfg5}  \\
     5 & \TT{N} \TT{[N]} \TT{NP}  & \TT{ate cheese} & complete \\
     6 & \TT{NP}  & \TT{ate cheese} & project \eqref{eq:cfg1}  \\
     7 & \TT{[VP]} \TT{S}  & \TT{ate cheese} & shift  \\
     8 & \TT{ate} \TT{[VP]} \TT{S}  & \TT{cheese} & project \eqref{eq:cfg6} \\
     9 & \TT{V} \TT{[VP]} \TT{S}   & \TT{cheese} & project \eqref{eq:cfg3}  \\
     10 & \TT{[N]} \TT{VP} \TT{[VP]} \TT{S}  & \TT{cheese} & shift \\
     11 & \TT{cheese} \TT{[N]} \TT{VP} \TT{[VP]} \TT{S}  & $\epsilon$ & project \eqref{eq:cfg7} \\
     12 & \TT{N} \TT{[N]} \TT{VP} \TT{[VP]} \TT{S}  & $\epsilon$ & complete \\
     13 & \TT{VP} \TT{[VP]} \TT{S}  & $\epsilon$ & complete \\
     15 & \TT{S}  & $\epsilon$ & accept \\
     \hline
   \end{tabular}
  \caption{Left-corner parser processing the example sentence \eqref{eq:sentence}. The stack expands to the left.}\label{tab:lcpars}
\end{table}

The left-corner parser shown in \Tab{tab:lcpars} essentially operates autonomously in modes project, complete and accept, but interactively in shift mode. Thus, we can significantly simplify the parsing process through a mapping from one intermediary automaton configuration to another one that is mediated by the interactively shifted input word \cite{Wegner98}. Expressing the configurations as temporary phrase structure trees yields then the symbolic computation in \Fig{fig:lctree}.

\begin{figure}[H]
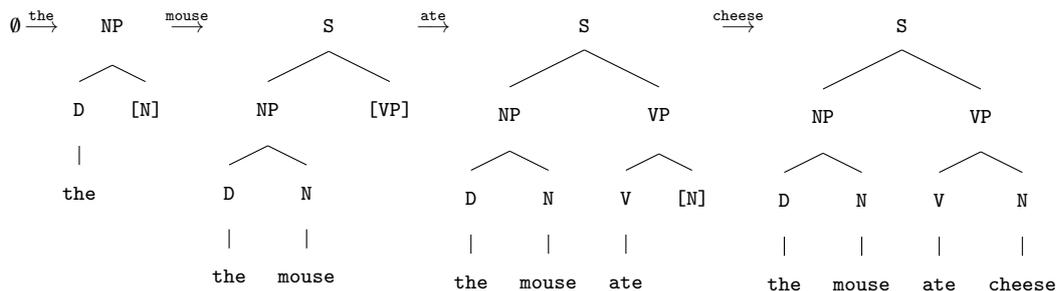

\[
 \emptyset
\stackrel{\TT{the}}{\longrightarrow}
\Tree
        [.\TT{NP}
            [.\TT{D} \TT{the}
            ].\TT{D}
            \TT{[N]}
        ].\TT{NP}
\stackrel{\TT{mouse}}{\longrightarrow}
\Tree
    [.\TT{S}
        [.\TT{NP}
            [.\TT{D} \TT{the}
            ].\TT{D}
            [.\TT{N} \TT{mouse}
            ].\TT{N}
        ].\TT{NP}
        \TT{[VP]}
    ].\TT{S}
\stackrel{\TT{ate}}{\longrightarrow}
\Tree
    [.\TT{S}
        [.\TT{NP}
            [.\TT{D} \TT{the}
            ].\TT{D}
            [.\TT{N} \TT{mouse}
            ].\TT{N}
        ].\TT{NP}
        [.\TT{VP}
            [.\TT{V} \TT{ate}
            ].\TT{V}
            \TT{[N]}
        ].\TT{VP}
    ].\TT{S}
\stackrel{\TT{cheese}}{\longrightarrow}
\Tree
    [.\TT{S}
        [.\TT{NP}
            [.\TT{D} \TT{the}
            ].\TT{D}
            [.\TT{N} \TT{mouse}
            ].\TT{N}
        ].\TT{NP}
        [.\TT{VP}
            [.\TT{V} \TT{ate}
            ].\TT{V}
            [.\TT{N} \TT{cheese}
            ].\TT{N}
        ].\TT{VP}
    ].\TT{S}
\]
  \caption{\label{fig:lctree} Interactive LC parse of the example sentence \eqref{eq:sentence}.}
\end{figure}

According to our previous definitions, the states of the processor are the automaton configurations in \Tab{tab:lcpars} or the temporary phrase structures trees in \Fig{fig:lctree}, that are both interpretable in terms of LC parsing and language processing for an informed expert observer. Moreover, the processing steps in the last column of \Tab{tab:lcpars} and also the interactive mappings \Fig{fig:lctree} are understandable and thereby explainable by the observer. In principle, one could augment the left-corner parser with a ``reasoning engine'' \cite{DoranSchulzBesold17} that translates the formal language used in those symbolic representations into everyday language. The result would be something like the (syntactic) ``meaning'' $\mng{w}$ of a word $w$ that can be regarded as the operator mapping a tree in \Fig{fig:lctree} to its successor. This interactive interpretation of meaning is well-known in \emph{dynamic semantics} \cite{Gardenfors88, GroenendijkStokhof91, Kracht02}. Therefore, symbolic AI is straightforwardly interpretable and explainable \cite{DoranSchulzBesold17}.

% ------------------------------------- Section -----------------------------------
\subsection{Algebraic Description}
\label{sec:alg}

In order to prepare the construction of a vector symbolic architecture (VSA) \cite{Smolensky90, Mizraji89, Plate95, GrabenPotthast09a, Kanerva09, Gayler06} in the next step, we need an algebraically more sophisticated description. This is provided by the concept of a \emph{term algebra} \cite{Kracht03}. A term algebra is defined over a \emph{signature} $\Sigma = (F, \rank)$ where $F$ is a finite set of function symbols and $\rank : F \to \mathbb{N}_0$ is an arity function, assigning to each symbol $f \in F$ an integer indicating the number of arguments that $f$ has to take.

To apply this idea to a CFG, we introduce a new kind of grammar normal form that we call \emph{term normal form} in the following. A CFG $G = (T, N, \mathtt{S}, R)$ is said to be in term normal form when for every category $A \in N$ holds: if $A$ is expanded into $n\in\mathbb{N}$ rules, $r_1 : A \to \gamma_1$ to $r_n : A \to \gamma_n$, then $|\gamma_1| = \ldots = |\gamma_n|$.

It can be easily demonstrated that every CFG can be transformed into a weakly equivalent CFG in term normal form, where weak equivalence means that two different grammars derive the same context-free language. A proof is presented in Appendix \ref{sec:tnf}.

Obviously, the rules (\ref{eq:cfg1} -- \ref{eq:cfg3}) of our example above are already in term normal form, simply because they are not ambiguous. Thus, we define a term algebra by regarding the set of variables $V = N \cup T$ as signature with arity function $\rank : V \to \mathbb{N}_0$ such that i) $\rank(a) = 0$ for all $a \in T$, i.e. terminals are nullary symbols and hence constants; ii) $\rank(A) = |\gamma|$ for categories $A \in N$, that are expanded through rules $A \to \gamma$. Moreover, when $G$ is given in Chomsky normal form, for all categories $A \in N$ appearing exclusively in lexical rules $\rank(A) = 1$, i.e. lexical categories (\TT{D}, \TT{N}, \TT{V}) are unary functions. Whereas, $\rank(A) = 2$ for all categories $A \in N$ that appear exclusively in syntactic rules, which are hence binary functions.

For a general CFG $G$ in term normal form, we define the term algebra $\frak{T}(G)$ inductively: i) every terminal symbol $a \in T$ is a term, $a \in \frak{T}(G)$. ii) Let $A \in N$ be a category with $\rank(A) = k$ and let $t_0, \dots, t_{k-1} \in \frak{T}(G)$ be terms, then $A(t_0, \dots, t_{k-1}) \in \frak{T}(G)$ is a term. Additionally, we want to describe LC phrase structure trees as well. To this end, we extend the signature by the predicted categories $P = \{ [\mathtt{N}], [\mathtt{VP}] \}$, that are interpreted as constants with $\rank(C) = 0$ for $C \in P$. The enlarged term algebra is denoted by $\frak{T}_\mathrm{LC}(G)$. We also allow for $\emptyset \in \frak{T}_\mathrm{LC}(G)$.

In the LC term algebra $\frak{T}_\mathrm{LC}(G)$, we encode the tree of step 1 in \Fig{fig:lctree} (beginning with the empty tree $t_0 = \emptyset$ in step 0) as term
\begin{equation}\label{eq:tm1}
  t_1 = \TT{NP}(\TT{D}(\TT{the}), \TT{[N]})
\end{equation}
because $\rank(\TT{NP}) = 2$, $\rank(\TT{D}) = 1$, and $\rank(\TT{the}) = \rank(\TT{[N]}) = 0$. Likewise we obtain
\begin{equation}\label{eq:tm2}
  t_2 = \TT{S}(\TT{NP}(\TT{D}(\TT{the}), \TT{N}(\TT{mouse})), \TT{[VP]})
\end{equation}
as the term representation of the succeeding step 2 in \Fig{fig:lctree}.

Next, we define several partial functions over $\frak{T}_\mathrm{LC}(G)$ as follows \cite{Smolensky90, Smolensky06}.
\begin{subequations}
\begin{align}
    \cat(A(t_0, \dots, t_k)) &= A \label{eq:term1} \\
    \ext_i(A(t_0, \dots, t_k)) &= t_i \label{eq:term2} \\
    \cons_k(A, t_0, \dots, t_k) &= A(t_0, \dots, t_k) \label{eq:term3} \:.
\end{align}
\end{subequations}
Here, the function $\cat : \frak{T}_\mathrm{LC}(G) \to N$ yields the category, i.e. the function symbol $A$ of the term $A(t_0, \dots, t_k) \in \frak{T}_\mathrm{LC}(G)$. The functions $\ext_i: \frak{T}_\mathrm{LC}(G) \to \frak{T}_\mathrm{LC}(G)$ for term extraction and $\cons_k: N \times \frak{T}_\mathrm{LC}(G)^{k + 1} \to \frak{T}_\mathrm{LC}(G)$ as term constructor are defined only partially, when $A(t_0, \dots,\allowbreak t_k) \in \dom(\ext_i)$, if $k = \rank(A) - 1$ and $i < k$, as well as $(A, t_0, \dots, t_k) \in \dom(\cons_k)$, if $k = \rank(A) - 1$.

By means of the term transformations (\ref{eq:term1} -- \ref{eq:term3}) we can express the action of an incrementally and interactively shifted word $a \in T$ through a term operator $\mng{a} : \frak{T}_\mathrm{LC}(G) \to \frak{T}_\mathrm{LC}(G)$. For the transition from, e.g., LC tree 1 to LC tree 2 in \Fig{fig:lctree} we obtain
\begin{equation}\label{eq:termop}
    \mng{\TT{mouse}}(t_1) =
    \cons_2(\TT{S}, \cons_2(\cat(t_1), \ext_0(t_1), \TT{N}(\TT{mouse})), \TT{[VP]}) = t_2 \:.
\end{equation}
Therefore, the (syntactic) meaning of the word ``\TT{mouse}'' is its impact on the symbolic term algebra.

% ------------------------------------- Section -----------------------------------
\subsection{Vector Symbolic Architectures}
\label{sec:vsa}

In vector-symbolic architectures (VSA) \cite{Smolensky90, Mizraji89, Plate95, GrabenPotthast09a, Kanerva09, Gayler06} hierarchically organized complex data structures are represented as vectors in high dimensional linear spaces. The composition of these structures is achieved by two basic operations: binding and bundling. While bundling is commonly implemented as vector superposition, i.e. addition, different VSA realize binding in particular ways: originally through tensor products \cite{Smolensky90, Mizraji89}, through circular convolution in reduced holographic representations (HRR) \cite{Plate95}, through XOR spatter code \cite{Kanerva94} or through Hadamard products \cite{LevyGayler08}. While HRR, spatter code, Hadamard products or a combination of tensor products with nonlinear compression \cite{Smolensky06} are lossy representations that require a clean-up module (usually an attractor neural network, cf. \cite{Kanerva09}), tensor product representations of basis vectors are faithful, thereby allowing interpretable and explainable VSA \cite{DoranSchulzBesold17}.

Coming back to our linguistic example, we construct a homomorphism $\psi : \frak{T}_\mathrm{LC}(G) \cup N \to \mathcal{F}$ from the term algebra unified with its categories $N$ to a vector space $\mathcal{F}$ in such a way, that the structure of the transformations (\ref{eq:term1} -- \ref{eq:term3}) is preserved. The resulting images $\psi(t)$ for terms $t \in \frak{T}_\mathrm{LC}(G)$ become vector space operators, i.e. essentially matrices acting on $\mathcal{F}$.

Again, we proceed inductively. First we map the symbols in $\frak{T}_\mathrm{LC}(G) \cup N$ onto vectors. To each atomic symbol $s \in T \cup N \cup P$ we assign a so-called \emph{filler} basis vector $\ket{s} = \psi(s) \in \mathcal{F}$, calling the subspace $\mathcal{V}_F = \spn(\psi(T \cup N \cup P))$ the filler space. Its dimension $n = \dim \mathcal{V}_F$ corresponds to the number of atomic symbols in $T \cup N \cup P$, which is $n = 13$ in our example.

Let further $m = \max(\{ |\gamma| \,|\, (A \to \gamma) \in R \})$ be the length of the largest production of grammar $G$. Then, we define $m + 1$ so-called \emph{role} vectors $\ket{i}$, spanning the role space $\mathcal{V}_R = \spn(\{ \ket{i} \,|\, 0 \le i \le m \})$. Note that we employ the so-called Dirac notation from quantum mechanics that allows a coordinate-free and hence representation-independent description here \cite{Dirac39}. Then, the role $\ket{0}$ denotes the 1st daughter node, $\ket{1}$ the 2nd daugther and so on, until the last daughter $\ket{m - 1}$. The remaining role $\ket{m}$ bounds the mother node in the phrase structure trees of grammar $G$. In our example, because $G$ has Chomsky normal form, we have $m = 2 = \dim \mathcal{V}_R - 1$ such that there are three roles for positions in a binary branching tree: left daughter $\ket{0}$, right daughter $\ket{1}$, and mother $\ket{2}$. For binary trees, we also use a more intuitive symbolic notation: left daughter $\ket{/}$, right daughter $\ket{{\setminus}}$, and mother $\ket{{\wedge}}$.

Let $A(t_0, \dots, t_k) \in \frak{T}_\mathrm{LC}(G)$ be a term. Then, we define the \emph{tensor product representation} of $A(t_0, \dots, t_k) \in \frak{T}_\mathrm{LC}(G)$ in vector space $\mathcal{F}$ recursively as follows
\begin{equation}\label{eq:tensor}
  \psi(A(t_0, \dots, t_k)) = \ket{A} \otimes \ket{m} \oplus \psi(t_0) \otimes \ket{0} \oplus \cdots
    \oplus \psi(t_k) \otimes \ket{m - 1} \:.
\end{equation}
As a shorthand notation, we suggest the Dirac expression
\begin{equation}\label{eq:tensor2}
  \ket{A(t_0, \dots, t_k)} = \ket{A} \otimes \ket{m} \oplus \ket{t_0} \otimes \ket{0} \oplus \cdots
    \oplus \ket{t_k} \otimes \ket{m - 1} \:.
\end{equation}

Here the symbol ``$\otimes$'' refers to the (Kronecker) tensor product, mapping two vectors onto another vector, in contrast to the dyadic (outer) tensor product, which yields a matrix, hence being a vector space operator. In addition, ``$\oplus$'' denotes the (outer) direct sum that is mandatory for the superposition of vectors from spaces with different dimensionality.

Obviously, the (in principle) infinite recursion of the mapping  $\psi$ leads to an infinite-dimensional representation space
\begin{equation}\label{eq:fockspace}
    \mathcal{F} =
    \bigoplus_{p = 0}^\infty \left( \mathcal{V}_F \otimes \mathcal{V}_R^{\otimes^p} \right) \oplus \mathcal{V}_R \:,
\end{equation}
that is known as \emph{Fock space} from quantum field theory \cite{GrabenPotthast09a, Fock32, Aerts09, Smolensky12}.

In quantum field theory, there is a distinguished state $\ket{\mathbf{0}} \ne 0$, the vacuum state, spanning a one-dimensional subspace, the vacuum sector that is isomorphic to the underlying number field. According to \eqref{eq:fockspace}, this sector is contained in the subspace spanned by filler and role spaces, $\mathcal{V}_F \oplus \mathcal{V}_R$. Therefore, we could represent the empty tree in \Fig{fig:lctree} by an arbitrary role; a suitable choice is the mother role $\psi(\emptyset) = \ket{m} \cong \ket{\mathbf{0}}$, hence symbolizing the vacuum state.

Using the tensor product representation \eqref{eq:tensor}, we can recursively compute the images of our example terms above. For \eqref{eq:tm1} we obtain
\begin{multline}\label{eq:tptm1}
    \ket{t_1} = \ket{\TT{NP}(\TT{D}(\TT{the}), \TT{[N]})} = \ket{\TT{NP}} \otimes \ket{2} \oplus \ket{\TT{D}(\TT{the})} \otimes \ket{0} \oplus \ket{\TT{[N]}} \otimes \ket{1} =\\
    \ket{\TT{NP}} \otimes \ket{2} \oplus
    ( \ket{\TT{D}} \otimes \ket{2} \oplus \ket{\TT{the}} \otimes \ket{0} ) \otimes \ket{0} \oplus \ket{\TT{[N]}} \otimes \ket{1} = \\
    \ket{\TT{NP}} \otimes \ket{2} \oplus
    \ket{\TT{D}} \otimes \ket{2} \otimes \ket{0} \oplus \ket{\TT{the}} \otimes \ket{0} \otimes \ket{0} \oplus \ket{\TT{[N]}} \otimes \ket{1} =\\
    \ket{\TT{NP} 2} \oplus
    \ket{\TT{D} 2 0} \oplus \ket{\TT{the} 0 0} \oplus \ket{\TT{[N]} 1} = \\
    \ket{\TT{NP} {\wedge}} \oplus
    \ket{\TT{D} {\wedge} /} \oplus \ket{\TT{the} / /} \oplus \ket{\TT{[N]} {\setminus}}  \:,
\end{multline}
where we used the compressed Dirac notation $\ket{a} \otimes \ket{b} = \ket{a b}$ in the last steps. The last line is easily interpretable in terms of phrase structure: It simply states that \TT{NP} occupies the root of the tree, \TT{D} appears as its immediate left daughter, \TT{the} is the left daughter's left daughter and a leave, and finally  \TT{[N]} is a leave bound to the right daughter of the root. Note that the Dirac kets have to be interpreted from the right to the left (reading the arabic manner). The vector $\ket{t_1}$ belongs to a Fock subspace of dimension
\begin{equation}\label{eq:fockdim}
  q  = n \frac{m^{p + 1} - 1}{m - 1} + m
\end{equation}
where $n = \dim(\mathcal{V}_F)$, $m = \dim(\mathcal{V}_R)$ and $p$ the embedding depth in the phrase structure tree step 1 of \Fig{fig:lctree}. This leads to $q_1 = 172$ for $\ket{t_1}$.

Similarly, we get for \eqref{eq:tm2}
\begin{multline}\label{eq:tptm2}
    \ket{t_2} = \ket{\TT{S}(\TT{NP}(\TT{D}(\TT{the}), \TT{N}(\TT{mouse})), \TT{[VP]})} = \\
    \ket{\TT{S}} \otimes \ket{2} \oplus \ket{\TT{NP}( \TT{D}( \TT{the} ), \TT{N}(\TT{mouse}))} \otimes \ket{0}  \oplus \ket{\TT{[VP]}} \otimes \ket{1} = \\
    \ket{\TT{S}} \otimes \ket{2} \oplus  ( \ket{\TT{NP}}  \otimes \ket{2} \oplus \ket{\TT{D}(\TT{the})} \otimes \ket{0} \oplus \ket{\TT{N}(\TT{mouse})} \otimes \ket{1} )  \otimes \ket{0}  \oplus \ket{\TT{[VP]}} \otimes \ket{1} = \\
    \ket{\TT{S}} \otimes \ket{2} \oplus \ket{\TT{NP}}  \otimes \ket{2} \otimes \ket{0}  \oplus
            \ket{\TT{D}(\TT{the})} \otimes \ket{0} \otimes \ket{0}  \oplus \ket{\TT{N}(\TT{mouse})} \otimes \ket{1} \otimes \ket{0} \oplus \ket{\TT{[VP]}} \otimes \ket{1} = \\
    \ket{\TT{S}} \otimes \ket{2} \oplus
         \ket{\TT{NP}}  \otimes \ket{2} \otimes \ket{0}  \oplus
            (\ket{\TT{D}} \otimes \ket{2} \oplus
                \ket{\TT{the}} \otimes \ket{0} ) \otimes \ket{0} \otimes \ket{0}  \oplus \\
                ( \ket{\TT{N}} \otimes \ket{2} \oplus \ket{\TT{mouse}} \otimes \ket{0} ) \otimes \ket{1} \otimes \ket{0} \oplus \ket{\TT{[VP]}} \otimes \ket{1} = \\
    \ket{\TT{S}} \otimes \ket{2} \oplus
         \ket{\TT{NP}}  \otimes \ket{2} \otimes \ket{0}  \oplus
            \ket{\TT{D}} \otimes \ket{2} \otimes \ket{0} \otimes \ket{0}  \oplus
                \ket{\TT{the}} \otimes \ket{0}  \otimes \ket{0} \otimes \ket{0}  \oplus \\
                \ket{\TT{N}} \otimes \ket{2} \otimes \ket{1} \otimes \ket{0} \oplus \ket{\TT{mouse}} \otimes \ket{0} \otimes \ket{1} \otimes \ket{0} \oplus \ket{\TT{[VP]}} \otimes \ket{1} = \\
    \ket{\TT{S} 2} \oplus \ket{\TT{NP} 2 0}  \oplus \ket{\TT{D} 2 0 0}  \oplus \ket{\TT{the} 0 0 0}  \oplus
                \ket{\TT{N} 2 1 0} \oplus \ket{\TT{mouse} 0 1 0} \oplus \ket{\TT{[VP]} 1} = \\
    \ket{\TT{S} {\wedge}} \oplus \ket{\TT{NP} {\wedge} /}  \oplus \ket{\TT{D} {\wedge} / /}  \oplus \ket{\TT{the} / / /}  \oplus
                \ket{\TT{N} {\wedge} {\setminus} /} \oplus \ket{\TT{mouse} / {\setminus} /} \oplus \ket{\TT{[VP]} {\setminus}} \:,
\end{multline}
where we have again utilized the more intuitive branching notation in the last line which can be straightforwardly interpreted in terms of tree addresses as depicted in \Fig{fig:lctree}~(step 2). Computing the dimension of the respective Fock subspace according to \eqref{eq:fockdim} yields $q_2 = 523$ for $\ket{t_2}$.

In Fock space, the interactive and incremental action of a word $a \in T$  is then represented as a matrix operator $\mng{a}_\psi : \mathcal{F} \to \mathcal{F}$. For the transition from \eqref{eq:tm1} to \eqref{eq:tm2} we obtain
\begin{multline}\label{eq:fockop}
    \mng{\TT{mouse}}_\psi \ket{t_1} =
    \mng{\TT{mouse}}_\psi (
    \ket{\TT{NP} {\wedge}} \oplus \ket{\TT{D} {\wedge} /} \oplus \ket{\TT{the} / /} \oplus \ket{\TT{[N]} {\setminus}} ) = \\
    \ket{\TT{S} {\wedge}} \oplus \ket{\TT{NP} {\wedge} /}  \oplus \ket{\TT{D} {\wedge} / /}  \oplus \ket{\TT{the} / / /}  \oplus
                \ket{\TT{N} {\wedge} {\setminus} /} \oplus \ket{\TT{mouse} / {\setminus} /} \oplus \ket{\TT{[VP]} {\setminus}} = \ket{t_2} \:.
\end{multline}

In order to prove $\psi$ a homomorphism, we define the following linear maps on $\mathcal{F}$.
\begin{subequations}
\begin{align}
    \bfcat(\ket{u}) &= (\1 \otimes \bra{m}) \ket{u} \label{eq:hom1} \\
    \bfext_i(\ket{u}) &=  (\1 \otimes \bra{i}) \ket{u} \label{eq:hom2} \\
    \bfcons_k(\ket{a}, \ket{u_0}, \dots, \ket{u_k}) &=
    \ket{a} \otimes \ket{m} \oplus \ket{u_0} \otimes \ket{0} \oplus \cdots \oplus \ket{u_k} \otimes \ket{k} \label{eq:hom3} \:,
\end{align}
\end{subequations}
here, $\1$ denotes the unit operator (i.e. the unit matrix) and the Dirac ``bra'' vectors $\bra{k}$ are linear forms from the dual role space $\mathcal{V}_R^*$ that are adjoined to the role ``ket'' vectors $\ket{k}$ such that $\braket{i}{k} = \delta_{ik}$ with Kronecker's $\delta_{ik} = 0(1)$ for $i \ne k (i = k)$.

By means of these homomorphisms we compute the meaning of ``\TT{mouse}'' as Fock space operator through
\begin{equation}\label{eq:termop}
    \mng{\TT{mouse}}_\psi \ket{t_1} =
    \bfcons_2(\ket{\TT{S}}, \bfcons_2(\bfcat(\ket{t_1}), \bfext_0(\ket{t_1}), \ket{\TT{N}(\TT{mouse})}), \ket{\TT{[VP]}}) = \ket{t_2} \:.
\end{equation}
Inserting (\ref{eq:hom1} -- \ref{eq:hom3})  yields
\begin{multline}\label{eq:fockop}
    \mng{\TT{mouse}}_\psi \ket{t_1} =
    \bfcons_2(\ket{\TT{S}}, \bfcons_2((\1 \otimes \bra{2}) \ket{t_1}, (\1 \otimes \bra{0}) \ket{t_1}, \ket{\TT{N}(\TT{mouse})}), \ket{\TT{[VP]}}) = \\
    \bfcons_2(\ket{\TT{S}},
        (\1 \otimes \bra{2}) \ket{t_1} \otimes \ket{2} \oplus
        (\1 \otimes \bra{0}) \ket{t_1} \otimes \ket{0} \oplus
        \ket{\TT{N}(\TT{mouse})} \otimes \ket{1}, \ket{\TT{[VP]}}) = \\
    \ket{\TT{S}} \otimes \ket{2} \oplus
        ((\1 \otimes \bra{2}) \ket{t_1} \otimes \ket{2} \oplus
        (\1 \otimes \bra{0}) \ket{t_1} \otimes \ket{0} \oplus
        \ket{\TT{N}(\TT{mouse})} \otimes \ket{1}) \otimes \ket{0} \oplus
         \ket{\TT{[VP]}}) \otimes \ket{1} = \\
    \ket{\TT{S}} \otimes \ket{2} \oplus
        ((\1 \otimes \bra{2}) \ket{t_1} \otimes \ket{2} \oplus
        (\1 \otimes \bra{0}) \ket{t_1} \otimes \ket{0} \oplus
        ( \ket{\TT{N}} \otimes \ket{2} \oplus \ket{\TT{mouse}} \otimes \ket{0} )
        \otimes \ket{1}) \otimes \ket{0} \oplus \\
         \ket{\TT{[VP]}}) \otimes \ket{1} = \ket{t_2}  \:,
\end{multline}
where we have expanded $\ket{\TT{N}(\TT{mouse})}$ as in \eqref{eq:tptm2} above. Note that the meaning of ``\TT{mouse}'' crucially depends on the given state $\ket{t_1}$ subjected to the operator $\mng{\TT{mouse}}_\psi$, making meaning highly contextual. This is an important feature of dynamic semantics as well \cite{Gardenfors88, GroenendijkStokhof91, Kracht02}.

% ------------------------------------- Section -----------------------------------
\section{Results}
\label{sec:res}

The main result of this study is a Fock space representation theorem for vector symbolic architectures of context-free grammars that follows directly from the definitions (\ref{eq:hom1} -- \ref{eq:hom3}) and is proven in Appendix \ref{sec:rethe}.

The tensor product representation $\psi : \frak{T}_\mathrm{LC}(G) \cup N \to \mathcal{F}$ is a homomorphism with respect to the term transformations (\ref{eq:term1} -- \ref{eq:term3}). It holds
\begin{subequations}
\begin{align}
    \bfcat(\ket{A(t_0, \dots, t_k)}) &= \ket{\cat(A(t_0, \dots, t_k))} \label{eq:theo1} \\
    \bfext_i(\ket{A(t_0, \dots, t_k)}) &= \ket{\ext_i(A(t_0, \dots, t_k))}  \label{eq:theo2} \\
    \bfcons_k(\ket{A}, \ket{t_0}, \dots, \ket{t_k})
    &= \ket{\cons_k(A, t_0, \dots, t_k)} \label{eq:theo3} \:.
\end{align}
\end{subequations}

For the particular example discussed above, we obtain the Fock space trajectory in \Tab{tab:fockpars}.

\begin{table}[H]
  \centering
  \begin{tabular}{llll}
     \hline
     \# & Fock vector & dim & operation \\
     \hline
     0 & $\ket{{\wedge}}$ & 16 & shift \TT{the} \\
     1 & $\ket{\TT{D}  {\setminus} {\wedge}  /  } \oplus \ket{\TT{NP}  {\setminus} {\wedge}   } \oplus \ket{\TT{[N]}  {\setminus}  } \oplus \ket{\TT{the}  /  /}$ & 172 & shift \TT{mouse} \\
     2 & $\ket{\TT{D}  {\setminus} {\wedge}  /  /    } \oplus     \ket{\TT{NP}  {\setminus} {\wedge}  /     } \oplus     \ket{\TT{N}  {\setminus} {\wedge}  {\setminus}  /    } \oplus     \ket{\TT{S}  {\setminus} {\wedge}        }\oplus     \ket{\TT{[VP]}  {\setminus}     }\oplus     \ket{\TT{mouse}  /  {\setminus}  /}\oplus     \ket{\TT{the}  /  /  /  }$ & 523 & shift \TT{ate} \\
     3 & $\ket{\TT{D}  {\setminus} {\wedge}  /  /    }\oplus     \ket{\TT{NP}  {\setminus} {\wedge}  /     }\oplus     \ket{\TT{N}  {\setminus} {\wedge}  {\setminus}  /    }\oplus     \ket{\TT{S}  {\setminus} {\wedge}        }\oplus     \ket{\TT{VP}  {\setminus} {\wedge}  {\setminus}     }\oplus  \ket{\TT{V}  {\setminus} {\wedge}  /  {\setminus}    }\oplus     \ket{\TT{[N]}  {\setminus}  {\setminus}    }\oplus    \ket{\TT{ate}  /  /  {\setminus}  }\oplus     \ket{\TT{mouse}  /  {\setminus}  /}\oplus     \ket{\TT{the}  /  /  /  }$  & 523 & shift \TT{cheese} \\
     4 & $\ket{\TT{D}  {\setminus} {\wedge}  /  /     }\oplus     \ket{\TT{NP}  {\setminus} {\wedge}  /      }\oplus     \ket{\TT{N}  {\setminus} {\wedge}  {\setminus}  /     }\oplus     \ket{\TT{N}  {\setminus} {\wedge}  {\setminus}  {\setminus}     }\oplus    \ket{\TT{S}  {\setminus} {\wedge}         }\oplus
    \ket{\TT{VP}  {\setminus} {\wedge}  {\setminus}      }\oplus     \ket{\TT{V}  {\setminus} {\wedge}  /  {\setminus}     }\oplus     \ket{\TT{ate}  /  /  {\setminus}   }\oplus     \ket{\TT{cheese}  /  {\setminus}  {\setminus}}\oplus     \ket{\TT{mouse}  /  {\setminus}  / }\oplus     \ket{\TT{the}  /  /  /   }$  & 523 & accept \\
     \hline
   \end{tabular}
  \caption{Fock space representation of LC parser processing the example sentence \eqref{eq:sentence}.}\label{tab:fockpars}
\end{table}

Moreover, we present the complete Fock space LC parse generated by FockBox which is a MATLAB toolbox provided by \citet{WolffWirschingEA18} as its three-dimensional projection after principal component analysis (PCA \cite{WolffWirschingEA18}) in \Fig{fig:pcatraj} as illustration.

\begin{figure}[H]
\centering
\includegraphics[width=0.7\textwidth]{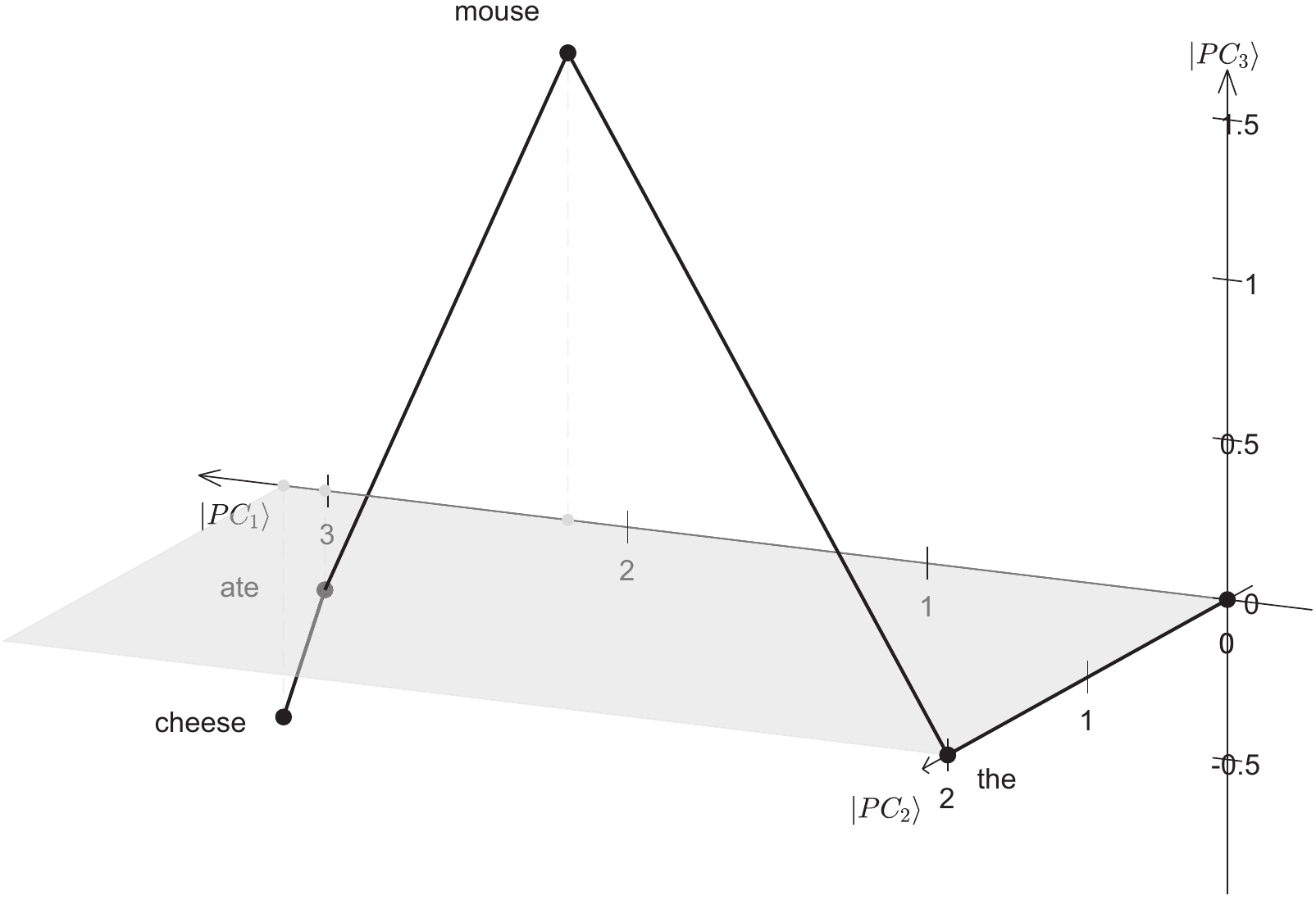}
 \caption{\label{fig:pcatraj} Principal component (PC) projection of the LC parser's Fock space representation. Shown are the first three PCs.}
\end{figure}

% ------------------------------------- Section -----------------------------------
\section{Discussion}
\label{sec:disc}

In this article we developed a representation theory for context-free grammars and push-down automata in Fock space as a vector symbolic architecture (VSA). We presented rigorous proofs for the representations of suitable term algebras. To this end, we suggested a novel normal form for CFG allowing to express CFG parse trees as terms over a symbolic term algebra. Rule-based derivations over that algebra are then represented as transformation matrices in Fock space.

Motivated by a seminal study of \citet{Shannon53} on cognitive dynamic systems \cite{Haykin12}, our work could be of significance for levering research on \emph{cognitive user interfaces} (CUI) \cite{Young10, HuberEA18}. Such systems are subject of ambitious current research. Instead of using keyboards and displays as input-output interfaces, users pronounce requests or instructions to a device as spoken language and listen to its uttered responses. To this aim, state-of-the-art language technology scans the acoustically analyzed speech signal for relevant keywords that are subsequently inserted into semantic frames \cite{Minsky74} to interpret the user's intent. This \emph{slot filling} procedure \cite{Allen03, TurHakkaniEA11, MesnilDauphinEA15} is based on large language corpora that are evaluated by machine learning methods, such as deep learning of neural networks \cite{CunBengioHinton15, Schmidhuber15, MesnilDauphinEA15}. The necessity to overcome traditional slot filling techniques by proper semantic analyses technologies has already been emphasized by \citet{Allen17}. His research group trains semantic parsers from large language data bases such as WordNet or VerbNet that are constrained by hand-crafted expert knowledge and semantic ontologies \cite{Allen03, AllenBahkshandehEA18}.

Another road toward realistic CUI systems is the development of utterance-meaning transducers (UMT) that map syntactic representations obtained from the speech signal onto semantic representations in terms of feature value relations (FVR) \cite{Karttunen84, HuberEA18}. This is achieved through a perception action cycle, comprising the three components: perception, action and behavior control. The perception module transforms the input from the signal layer to the semantic symbolic layer, the module for behavior control solves decision problems based on semantic information and computes appropriate actions. Finally, the action module executes the result by producing acoustic feedback. Behavior control can flexibly adapt to user's demands through reinforcement learning.

For the implementation of rule-based symbolic computations in cognitive dynamic systems, such as neural networks, VSA provide a viable approach. Our results contribute a formally sound basis for this kind of future research and engineering. In contrast to current black-box approaches, our method is essentially transparent and hence explainable and trustworthy \cite{Marcus20, DoranSchulzBesold17}.

% ------------------------------------- Section -----------------------------------
\section{Conclusion}
\label{sec:conc}

We reformulated context-free grammars (CFG) through term algebras and their processing through push-down automata by partial functions over term algebras. We introduced a novel normal form for CFG, called term normal form, and proved that any CFG in Chomsky normal form can be transformed into term normal form. Finally, we introduced a vector symbolic architecture (VSA) by assigning basis vectors of a high-dimensional linear space to the respective symbols and their roles in a phrase structure tree. We suggested a recursive function for mapping CFG phrase structure trees onto representation vectors in Fock space and proved a representation theorem for the partial rule-based processing functions. We illustrated our findings by an interactive left-corner parser and used FockBox, a freely accessible MATLAB toolbox, for the generation and visualization of Fock space VSA. Our approach directly encodes symbolic, rule-based knowledge into the hyperdimensional computing framework of VSA and can thereby supply substantial insights into the future development of explainable artifical intelligence (XAI).

%\begin{acknowledgements}
%\end{acknowledgements}

\section*{Compliance with Ethical Standards}

Ethical approval: This article does not contain any studies with human participants or animals performed by any of the authors.

Conflict of interest: The authors declare that they have no conflict of interest.

% Vancouver style
\bibliographystyle{unsrtnat}

% \bibliography{PbG}

% ------------------------------------- Section -----------------------------------
\section{Appendix}
\label{sec:apx}

% ------------------------------------- Section -----------------------------------
\subsection{Proof of term normal form}
\label{sec:tnf}

\begin{definition}[Context-free grammar]\label{def:cfg}
  A \emph{context-free grammar (CFG)} is a quadruple \(G = (T, N, \mathtt{S},
  R)\) with a set of \emph{terminals} \(T\), a set of \emph{nonterminals}
  \(N\), the \emph{start symbol} \(\mathtt{S}\in N\) and a set of rules
  \(R\subseteq N\times (N \cup T)^*\)\@. A rule \(r=(A, \gamma)\in R\)
  is usually written as a production \(r:A\to \gamma\)\@.
\end{definition}

\begin{definition}[Chomsky normal form]\label{def:cnf}
  According to \cite{HopcroftUllman79} a CFG \(G = (T, N, \mathtt{S}, R)\) is
  said to be in \emph{Chomsky normal form} iff every production \(r \in R\) is one of
  \begin{subequations}
    \begin{eqnarray}
      A & \to & B \, C\\
      A & \to & a\\
      \mathtt{S} & \to & \epsilon\label{eq:empty-string-production}
    \end{eqnarray}
  \end{subequations}
  with \(A\in N\),
  \(B, C\in N\setminus \{\mathtt{S}\}\) and
  \(a\in T\)\@.
\end{definition}

It is a known fact, that for every CFG \(G\) there is an equivalent CFG \(G'\) in Chomsky normal
form \cite{HopcroftUllman79}\@. It is also known that if \(G\) does not produce the empty string
--- absence of production \eqref{eq:empty-string-production} --- then there is an equivalent CFG
\(G'\) in Chomsky reduced form \cite{HopcroftUllman79}\@.

\begin{definition}[Chomsky reduced form]\label{def:crf}
  A CFG \(G = (T, N, \mathtt{S}, R)\) is said to be in \emph{Chomsky reduced
    form} iff every production \(r\in R\) is one of
  \begin{subequations}
    \begin{eqnarray}
      A & \to & B \, C \label{eq:length-two-production}\\
      A & \to & a\label{eq:length-one-production}
    \end{eqnarray}
  \end{subequations}
  with \(A, B, C\in N\) and \(a\in T\)\@.
\end{definition}

By utilizing some of the construction steps for establishing Chomsky normal form from
\cite{HopcroftUllman79} we deduce
\begin{corollary}\label{cor:crf-cnf}
  For every CFG \(G\) in Chomsky reduced form there is an equivalent CFG \(G'\) in Chomsky
  normal form without a rule corresponding to production \eqref{eq:empty-string-production}\@.
\end{corollary}

\begin{proof}
  Let \(G\) be a CFG in Chomsky reduced form\@. Clearly \(G\) does not produce the empty
  string\@. The only difference to Chomsky normal form is the allowed presence of the start
  symbol \(\mathtt{S}\) on the right-hand side of rules in \(R\)\@. By introducing a new start
  symbol \(\mathtt{S}_{0}\) and inserting rules
  \(\{(\mathtt{S}_{0}, \gamma)\mid \exists (\mathtt{S}, \gamma)\in R\}\) we eliminate this
  presence and obtain an equivalent CFG in Chomsky normal form without a production of form
  \eqref{eq:empty-string-production}\@. \qed
\end{proof}

\begin{definition}[Term normal form]\label{def:tnf}
  A CFG \(G = (T, N, \mathtt{S}, R)\) is said to be in \emph{term normal form}
  iff \(R\subseteq N\times (N \cup T)^{+}\) and for every two rules
  \(r=(A, \gamma)\in R\) and \(r'=(A', \gamma')\in R\)
  \begin{equation*}
    A=A'\implies |\gamma|=|\gamma'|
  \end{equation*}
  holds\@.
\end{definition}

We state and proof by construction:
\begin{theorem}\label{thm:cfg-tnf}
  For every CFG \(G = (T, N, \mathtt{S}, R)\) not producing the empty string
  there is an equivalent CFG \(G'\) in term normal form\@.
\end{theorem}

\begin{proof}
  Let \(G = (T, N, \mathtt{S}, R)\) be a CFG not producing the empty
  string\@. Let \(G' = (T, N', \mathtt{S}, R')\) be the equivalent CFG in
  Chomsky reduced form and \(D\subseteq N'\) be the set of all nonterminals
  from \(G'\) which have productions of both forms \eqref{eq:length-two-production} and
  \eqref{eq:length-one-production}\@.

  We establish term normal form by applying the following transformations to \(G'\):
  \begin{enumerate}
  \item For every nonterminal \(A\in D\) let
    \(R_{A}''=\{(A, B \, C)\in R'\mid B,C\in N'\}\) be the rules
    corresponding to productions of form \eqref{eq:length-two-production} and
    \(R_{A}'=\{(A, a)\in R'\mid a\in T\}\) be the rules
    corresponding to productions of form \eqref{eq:length-one-production}\@. We add
    \begin{enumerate}
    \item new nonterminals \(A''\) and \(A'\),
    \item a new rule \((A'', B \, C)\) for every rule
      \((A, B \, C)\in R_{A}''\) and
    \item a new rule \((A', a)\) for every rule
      \((A, a)\in R_{A}'\)\@.
    \end{enumerate}
    Finally, we remove all rules \(R_{A}''\cup R_{A}'\) from \(R'\)\@.
  \item For every nonterminal \(A\in D\) let
    \(L_{A}=\{(X, A \, Y)\in R'\mid X, Y\in
    N'\}\) be the set of rules where \(A\) appears at first position on the
    right-hand side\@. For every rule \((X, A \, Y)\in L_{A}\) we add
    \begin{enumerate}
    \item a new rule \((X, A'' \, Y)\) and
    \item a new rule \((X, A' \, Y)\)\@.
    \end{enumerate}
    Finally, we remove all rules \(L_{A}\) from \(R'\)\@.
  \item For every nonterminal \(A\in D\) let
    \(R_{A}=\{(X, Y \, A)\in R'\mid X, Y\in
    N'\}\) be the set of rules where \(A\) appears at second position on the
    right-hand side\@. For every rule \((X, Y \, A)\in R_{A}\) we add
    \begin{enumerate}
    \item a new rule \((X, Y \, A'')\) and
    \item a new rule \((X, Y \, A')\)\@.
    \end{enumerate}
    Finally, we remove all rules \(R_{A}\) from \(R'\)\@.
  \item\label{step:S-in-D} If \(\mathtt{S}\in D\) then we add
    \begin{enumerate}
    \item a new start symbol \(\mathtt{S}_{0}\),
    \item a new rule \((\mathtt{S}_{0}, \mathtt{S}')\) and
    \item a new rule \((\mathtt{S}_{0}, \mathtt{S}'')\)\@.
    \end{enumerate}
  \item Finally, we remove \(D\) from \(N'\)\@. \qed
  \end{enumerate}
\end{proof}

We immediately deduce
\begin{corollary}\label{cor:l1l2->tnf+cnf}
  For every CFG \(G\) only producing strings of either exactly length \(1\) or at least length
  \(2\) there is an equivalent CFG \(G'\) in term normal form which is also in Chomsky normal
  form\@.
\end{corollary}

\begin{proof}
  We handle the two cases separately\@.
  \begin{case}
    Let \(G\) be a CFG producing strings of exactly length \(1\)\@. Since \(G\) does not
    produce the empty string there is an equivalent CFG \(G'\) in Chomsky reduced form where
    every rule is of form \eqref{eq:length-one-production} and the only nonterminal being the
    start symbol\@. Obviously, \(G'\) is in Chomsky normal form and also in term normal form\@.
  \end{case}
  \begin{case}
    Let \(G\) be a CFG producing strings of at least length \(2\)\@. Since \(G\) does not
    produce the empty string there is an equivalent CFG in Chomsky reduced form and from
    corollary~\ref{cor:crf-cnf} follows that there is an equivalent CFG in Chomsky normal
    form\@. Applying the construction from theorem~\ref{thm:cfg-tnf} to this CFG leads to a CFG
    \(G'\) in term normal formal\@. Since \(G\) does not produce strings of length \(1\)
    step~\ref{step:S-in-D} is omitted by the construction and \(G'\) stays in Chomsky normal
    form\@. \qed
  \end{case}
\end{proof}

We also state the opposite direction\@.

\begin{corollary}\label{cor:tnf+cnf->l1l2}
  Every CFG \(G\) for which an equivalent CFG \(G'\) in Chomsky normal form exists which is also
  in term normal form, produces either only strings of length \(1\) or at least of length
  \(2\)\@.
\end{corollary}

\begin{proof}
  Let \(G = (T, N, \mathtt{S}, R)\) be a CFG in Chomsky normal form and term
  normal form at the same time\@. Clearly, \(G\) does not produce the empty string\@. Let
  \(R|_{\mathtt{S}}\subseteq R\) be the set of rules with the start symbols \(\mathtt{S}\) on the
  left side\@. Since \(G\) is in term normal form we have to consider the following two cases\@.
  \begin{case}
    Let \((\mathtt{S}, \gamma)\in R\) be a rule where \(\gamma\in T\)\@. Then every
    rule in the set \(R|_{\mathtt{S}}\) has to be of the same form\@. It follows that \(G\) only
    produces strings of length \(1\)\@.
  \end{case}
  \begin{case}
    Let \((\mathtt{S}, A \, B)\in R\) be a rule with \(A,B\in \mathtt{N}\)\@. Then every rule in
    the set \(R|_{\mathtt{S}}\) has to be of the same form\@. It follows that strings produced by
    \(G\) have to be at least of length \(2\)\@. \qed
  \end{case}
\end{proof}

We instantly deduce
\begin{theorem}
  Those CFGs for which a Chomsky normal form in term normal exists are exactly the CFGs
  producing either only strings of length \(1\) or strings with at least length \(2\)\@.
\end{theorem}
which follows directly from corollaries~\ref{cor:l1l2->tnf+cnf} and~\ref{cor:tnf+cnf->l1l2}\@.

% ------------------------------------- Section -----------------------------------
\subsection{Proof of representation theorem}
\label{sec:rethe}

The proof of the Fock space representation theorem for vector symbolic architectures follows from direct calculation using the definition of the tensor product representation \eqref{eq:tensor2}.

\begin{proof}

\begin{multline*}
    \bfcat(\ket{A(t_0, \dots, t_k)}) = (\1 \otimes \bra{m}) \ket{A(t_0, \dots, t_k)} = \\
    (\1 \otimes \bra{m})(\ket{A} \otimes \ket{m} \oplus \ket{t_0} \otimes \ket{0} \oplus \cdots \oplus \ket{t_k} \otimes \ket{k}) = \ket{A} = \ket{\cat(A(t_0, \dots, t_k))} \:,
\end{multline*}

\begin{multline*}
    \bfext_i(\ket{A(t_0, \dots, t_k)}) = (\1 \otimes \bra{i}) \ket{A(t_0, \dots, t_k)} = \\
    (\1 \otimes \bra{i})(
    \ket{A} \otimes \ket{m} \oplus \ket{t_0} \otimes \ket{0} \oplus \cdots \oplus \ket{t_k} \otimes \ket{k} ) = \ket{t_i} = \ket{\ext_i(A(t_0, \dots, t_k))} \:,
\end{multline*}

\begin{multline*}
    \bfcons_k(\ket{A}, \ket{t_0}, \dots, \ket{t_k}) =
    \ket{A} \otimes \ket{m} \oplus \ket{t_0} \otimes \ket{0} \oplus \cdots \oplus \ket{t_k} \otimes \ket{k} = \\
    \ket{A(t_0, \dots, t_k)} = \ket{\cons_k(A, t_0, \dots, t_k)}
\end{multline*}
\qed

\end{proof}

\end{document}